\documentclass[12pt]{article}
\RequirePackage{amssymb}
\RequirePackage{amsmath}
\usepackage{times,url,mathrsfs}
\usepackage{amsfonts}
\usepackage{graphicx}
\usepackage{subfigure}
\usepackage{arydshln}

\usepackage{amsthm}

\usepackage[default]{jasa_harvard}    

\RequirePackage{multirow}

\numberwithin{equation}{section}
\newtheorem{theorem}{Theorem}[section]
\newtheorem{lemma}{Lemma}[section]
\newtheorem{proposition}{Proposition}[section]

\theoremstyle{definition}
\newtheorem{definition}{Definition}[section]
\newtheorem{example}{Example}[section]

\newcommand{\bsb}{\boldsymbol}
\newcommand{\bsbX}{{\boldsymbol{X}}}
\newcommand{\bsbx}{{\boldsymbol{x}}}
\newcommand{\bsbY}{{\boldsymbol{Y}}}
\newcommand{\bsby}{{\boldsymbol{y}}}
\newcommand{\bsbb}{{\boldsymbol{\beta}}}

\newcommand{\bsbH}{{\boldsymbol{H}}}
\newcommand{\bsbI}{{\boldsymbol{I}}}

\newcommand{\bsbZ}{{\boldsymbol{Z}}}
\newcommand{\bsbz}{{\boldsymbol{z}}}

\newcommand{\bsbSig}{{\boldsymbol{\Sigma}}}
\newcommand{\bsbxi}{{\boldsymbol{\xi}}}
\newcommand{\bsbD}{{\boldsymbol{D}}}

\newcommand{\bsbU}{{\boldsymbol{U}}}
\newcommand{\bsbV}{{\boldsymbol{V}}}

\newcommand{\bsbA}{{\boldsymbol{A}}}
\newcommand{\bsba}{{\boldsymbol{\alpha}}}
\newcommand{\bsbB}{{\boldsymbol{B}}}

\newcommand{\bsbC}{{\boldsymbol{C}}}
\newcommand{\bsbc}{{\boldsymbol{c}}}

\newcommand{\bsbDelta}{{\boldsymbol{\Delta}}}
\newcommand{\bsbG}{{\boldsymbol{G}}}

\newcommand{\bsbS}{{\boldsymbol{S}}}

\newcommand{\bsbmu}{{\boldsymbol{\mu}}}
\newcommand{\bsbW}{{\boldsymbol{W}}}
\newcommand{\bsbM}{{\boldsymbol{M}}}

\newcommand{\bsbmI}{{\boldsymbol{\mathcal{I}}}}

\newcommand{\rd}{\,\mathrm{d}}

\begin{document}

\title{Reduced Rank Vector Generalized Linear Models for Feature Extraction}
\author{Yiyuan She
\\ Department of Statistics, Florida State University\\
yshe@stat.fsu.edu}
\date{}
\maketitle

\begin{abstract}
Supervised linear feature extraction can be achieved by fitting a reduced rank multivariate model. This paper studies rank penalized and rank constrained vector generalized linear models. From the perspective of thresholding rules, we build a framework for fitting singular value penalized models and use it for feature extraction. Through solving the rank constraint form of the problem, we propose progressive feature space reduction for fast computation in high dimensions with little performance loss. A novel projective cross-validation  is proposed for parameter tuning in such nonconvex setups. Real data applications are given to show the power of the methodology in supervised dimension reduction and feature extraction.
\end{abstract}



\section{Background}
\label{sec:bg}
Recently, high dimensional data analysis attracts a great deal of interest from statisticians. The availability of  a large pool of variables (relative to the sample size) poses  challenges in statistical modeling because in this high-dimensional setup, both estimation accuracy and model interpretability can be seriously hurt. Dimension reduction is a  natural and effective means to reduce the number of unknowns.  One can remove nuisance and/or redundant variables, referred to as  variable/feature selection; alternatively, one can find low dimensional linear or nonlinear projections of the input data, referred to as feature extraction. In this paper, we focus on linear feature extraction for dimension reduction purposes.

The most popular approach for linear feature extraction is perhaps the principle component analysis (PCA). Given $\bsbX \in \mathbb R^{n\times p}$ with $n$ observations and $p$ features, perform the Singular Value Decomposition (SVD) on the data  $\bsbX=\bsbU \bsbD \bsbV^T$. Given any $1\leq r \leq {rank}(\bsbX)$,  denote by $\bsbV_r$ the submatrix of $\bsbV$ consisting of its first $r$ columns. Then $$\bsbZ_r = \bsbX \bsbV_r$$ constructs $r$ new features as linear combinations of the original features.  The  extraction is optimal in the sense that  $\hat\bsbB = \bsbX {\mathcal P}_{\bsbV_r} = \bsbX \bsbV_r \bsbV_r^T$, with   ${\mathcal P}_{\bsbV_r}$ being the projection matrix onto the column space of $\bsbV_r$,   gives the best rank-$r$ approximation to $\bsbX$:
\begin{align*}
\hat \bsbB = \arg\min_{\bsbB: {rank}(\bsbB) \leq r} \|\bsbX - \bsbB\|_F^2,
\end{align*}
where   $\|\cdot\|_F$ is the Frobenius norm. A by-product is that the gram matrix of $\bsbX \bsbV_r$ is diagonal, 
which means all new features are uncorrelated with each other.

On the other hand, PCA is unsupervised. In many statistical learning problems, we want to construct  new features that best predict the responses. Suppose $\bsbY\in \mathbb R^{n\times m}$ is the response matrix, $n$ being the sample size and $m$ being the number of response variables. Supervised feature extraction can be  given by the \emph{reduced rank regression} (RRR) \cite{And51}, with the RRR estimator $\hat \bsbB$    defined by
\begin{align*}
\hat \bsbB = \arg \min_{\bsbB: {rank}(\bsbB) \leq r} \|\bsbY - \bsbX\bsbB\|_F^2.
\end{align*}
Assume $\bsbX$ has full column rank and  define $\bsbH=\bsbX (\bsbX^T \bsbX)^{-1} \bsbX^T$. Then $\hat \bsbB = \hat \bsbB_{{ols}} \bsbV_r \bsbV_r^T$, where $\hat \bsbB_{ols}= (\bsbX^T \bsbX)^{-1} \bsbX^T \bsbY$.   $\bsbV_r$ is formed by  the first $r$ columns of $\bsbV$ from the spectral decomposition  $\bsbY^T \bsbH \bsbY = \bsbV \bsbD \bsbV^T$. See \citeasnoun{ReinVelu}  for more details.  Therefore,
$$\bsbZ_r = \bsbX (\hat \bsbB_{ols} \bsbV_r)$$ constructs $r$  new   features that best approximate  $\bsbY$ in  Frobenius norm , and these new features are, again, uncorrelated.
The RRR framework includes the PCA as a special case, by setting $\bsbY=\bsbX$ \cite{izenbook}.
%
%
%
The above RRR solves a \emph{nonconvex} optimization problem in the classical setup $n>p$. Recently, \citeasnoun{BuneaSheWegkamp} studied the problem under $p>n$ and developed finite-sample theories as well as a computational algorithm.

On the other hand,  the squared error loss may not always  be appropriate. For vector generalized linear models (GLMs), such as   discrete responses   arising in classification problems,  deviance loss is much more reasonable.

Although there is a large body of literature on the RRR---\citeasnoun{Robinson74},  \citeasnoun{Rao79}, and \citeasnoun{brill}, to name a few, to the best of our knowledge, there is very little work beyond the Gaussian model.
\citeasnoun{yee} studied the reduced-rank vector GLM problem and used an iterative approximate estimation by fitting RRR repeatedly. Yet this only provides an approximation solution  to the original problem and there is not guarantee of  converge.
\citeasnoun{copulas} resorted to a continuation 
technique to deal with discrete responses.

This paper aims to  tackle the penalized and constrained
vector GLMs
\begin{align}
\min_{\bsbB} &  -\mbox{log-likelihood}(\bsbB; \bsbY, \bsbX) + \frac{\lambda^2}{2} rank(\bsbB), \mbox{ and} \label{penform}\\
\min_{\bsbB} & -\mbox{log-likelihood}(\bsbB; \bsbY, \bsbX) \quad  \mbox{ s.t. } rank(\bsbB) \leq r. \label{conform}
\end{align}
The imposed reduced rank structure is based on the belief that the  features' relevant directions, {in response to} $\bsbY$,  define  a lower  dimensional subspace in $\mathbb R^p$. The rank of such an estimator determines the number of new features to construct.
These two problems are \emph{not} equivalent to each other due to their nonconvexity. In fact, the rank function is nonconvex and discrete (and thus nondifferentiable), thereby posing a  challenge  in optimization. Our  algorithms boil down to an iterative version, which is not surprising in the GLM setup.

The rest of the paper is organized as follows.
Section 2 starts by studying a  matrix approximation problem, and then builds a framework for fitting  singular-value penalized multivariate GLMs. Supervised feature extraction can be attained for non-Gaussian models not necessarily using  the squared error loss.   The framework covers  a wide family of penalty functions.  A new parameter tuning strategy is proposed.  Section 3 tackles the rank constrained GLM problem and comes up with a feature space reduction technique.
Through this, \eqref{penform} and \eqref{conform} can be combined to achieve better estimation accuracy and computational efficiency.
Section 4 illustrates real applications of the proposed methodology. We conclude in Section 5. All technical details are left to the Appendix.

\section{Penalized Vector GLMs  for Feature Extraction}
\label{sec:penGLM}
In this section, we study the \emph{penalized} form reduced rank GLMs \eqref{penform}.
Our algorithm and analysis apply to  $p>n$ situations and cover a large family of singular-value penalties, including   nuclear norm, Frobenius norm, Schatten $\mathrm p$-penalties ($0<\mathrm p < 1$), and rank penalty. To  achieve such generality,  we start by studying a simpler   matrix approximation problem.
\subsection{Singular-value penalized matrix approximation}
\label{subsec:matapprox}
We consider the problem of matrix approximation  with a singular value penalty
\begin{eqnarray}
\min_{\bsbB} \frac{1}{2}\|\bsbY- \bsbB\|_F^2 + \sum_i P(\sigma_{i}^{(\bsbB)}; \lambda) \label{oriprob_mat}
\end{eqnarray}
where $\sigma_{i}^{(\bsbB)}$ denote the singular values of $\bsbB$.
The choice of the penalty function $P$ is flexible. For example, $P(t;\lambda)=\lambda |t|$ gives a multiple of the sum of singular values corresponding to the trace norm or nuclear norm penalty.   When $P(t;\lambda)= {\lambda^2} 1_{t\neq 0}/2$, we get the rank penalty which is discrete and nonconvex. 
For a general $P$, the closed-form solution to \eqref{oriprob_mat} is not known (to the best of our knowledge).
We address the problem from the standpoint of threshold functions.

\begin{definition}[Threshold function]\label{def:threshold}
A threshold function is a real valued
function  $\Theta(t;\lambda)$ defined for $-\infty<t<\infty$
and $0\le\lambda<\infty$ such that
\begin{enumerate}
\item $\Theta(-t;\lambda)= -\Theta(t;\lambda)$,
\item $\Theta(t;\lambda)\le \Theta(t';\lambda)$ for $t\le t'$, 
\item $\lim_{t\to\infty} \Theta(t;\lambda)=\infty$,\quad and
\item $0\le \Theta(t;\lambda)\le t$\ for\ $0\le t<\infty$.
\end{enumerate}
\end{definition}
\emph{Remarks. }
(i)
A vector version of $\Theta$ (still denoted by $\Theta$) is defined componentwise if
either $t$ or $\lambda$ is replaced by a vector.
(ii)
There may be some ambiguity in defining a threshold function. For example,
the hard-thresholding can be  defined as $\Theta_H(t; \lambda)=t 1_{|t|>\lambda}$ or $\Theta_H(t; \lambda)=t 1_{|t|\geq \lambda}$.  Fortunately, commonly used thresholding rules have at most finitely many discontinuity points and such discontinuities rarely occur in real data. When applying $\Theta$ to a quantity (say  $t$), we always make the implicit assumption that $\Theta$ is continuous at $t$.
(iii)
By definition,   $\Theta^{-1}(u;\lambda)\triangleq \sup\{t:\Theta(t;\lambda)\leq u\}, \forall u > 0$ must be monotonically increasing  on $(0, \infty)$ and  bounded between the identity line and $u=0$;  its derivative is defined almost everywhere  on $(0, \infty)$.
We assume that $$\rd \Theta^{-1}(u;\lambda)/\rd u\geq 1-{\mathcal L}_\Theta \quad \mbox{a.e. on } (0, \infty)$$ for some constant ${\mathcal L}_\Theta\in[0, 1]$ independent of $\lambda$. (In fact, for all \emph{convex} penalties constructed through \eqref{constrP}, ${\mathcal L}_{\Theta}$ can be set to be $0$.)

Next we introduce the matrix thresholding.
\begin{definition}[Matrix threshold function]\label{def:thresholdmat}
Given any  threshold function $\Theta(\cdot; \lambda)$, its matrix version $\Theta^{\sigma}$ is defined as follows
\begin{eqnarray}
\Theta^{\sigma}(\bsbB;\lambda)\triangleq \bsbU\mbox{diag}\{\Theta(\sigma_{i}^{(\bsbB)}; \lambda)\} \bsbV^T, \quad \forall \bsbB \in \mathbb R^{n\times m}\label{matthetadef}
\end{eqnarray}
where $\bsbU$, $\bsbV$, and $\mbox{diag}\{\sigma_{i}^{(\bsbB)}\}$ are obtained from the SVD of $\bsbB$: $\bsbB=\bsbU^T \mbox{diag}\{\sigma_{i}^{(\bsbB)}\} \bsbV$.
\end{definition}
Note that $\Theta(0; \lambda)=0$ by definition, and $\Theta^\sigma(\bsbB;\lambda)$ is not affected by the ambiguity of the SVD form. 


\begin{proposition}
\label{uniqsol-mat}
Given an arbitrary thresholding rule $\Theta$, let $P$ be any function satisfying
\begin{align}
P(\theta;\lambda)-P(0;\lambda)= \int_0^{|\theta|} (\sup\{s:\Theta(s;\lambda_k)\leq u\} - u) \rd u
 + q(\theta;\lambda), \label{constrP}
\end{align}
where $q(\cdot;\lambda)$ is nonnegative and $q(\Theta(t;\lambda); \lambda)=0$ for all $t\in \mathbb R$.  Then, the singular-value penalized minimization
\begin{eqnarray}
\min_\bsbB  F(\bsbB) = \|\bsbY-\bsbB\|_F^2/2 + \sum P(\sigma_{i}^{(\bsbB)}; \lambda)  \label{oriprob-nodesign}
\end{eqnarray}
has a unique optimal solution $\hat \bsbB=\Theta^\sigma(\bsbY;\lambda)$ for every $\bsbY$, provided $\Theta(\cdot;\lambda)$ is continuous at any singular value of $\bsbY$.
\end{proposition}
See Appendix \ref{approofMatApp} for its proof.

The function $q$ is often just zero, but can be nonzero in certain cases. In fact, we can use  a nontrivial $q$ to attain the exact  rank penalty; see \eqref{defofq}.
The proposition implies multiple (infinitely many, as a matter of fact) penalties can result in the same solution, which justifies our thresholding launching point  (rather than a penalty one). Some examples of the penalty $P$ and the coupled  $\Theta$ are listed in Table \ref{tab:pthetaexs}.

\begin{table}[ht]
\centering
\caption{\small{Some basic singular-value penalties and their coupled thresholding functions.} }\label{tab:pthetaexs}

\small{
\begin{tabular}{l  c c c c}
\hline

\hline
 & Nuclear norm & Frobenius  & Rank & Schatten-$\mathrm p$, $\mathrm p \in (0,1)$   \\
\hline
Penalty function & $\lambda\|\bsbB\|_*=\lambda \sum \sigma_i^{(\bsbB)}$ & $\lambda\|\bsbB\|_F^2$ & $\frac{\lambda^2}{2} rank(B)$ & $\lambda\sum  (\sigma_i^{(\bsbB)}  )^{\mathrm p}$  \\
{Thresholding rule} &
$(t - \mbox{sgn}(t)\lambda) 1_{|t|>\lambda}$ & $\frac{t}{1+\lambda}$ & $t 1_{|t|>\lambda}$ &   Ex 2.7 in  \citeasnoun{SheGLMTISP} \\
& (soft) & (ridge) & (hard) &   \\
\hline

\hline
\end{tabular}
}
\end{table}						

We point out two special cases of Proposition \ref{uniqsol-mat} as follows.
\paragraph*{A fusion between nuclear norm and Frobenius norm}
Define  a continuous thresholding rule
\begin{eqnarray}
\Theta_B(t; \lambda, M) = \begin{cases}0, & \mbox{ if } |t| \leq \lambda\\
t-\lambda\mbox{sgn}(t), & \mbox{ if }\lambda<|t|<\lambda+M\\
\frac{t}{1+\frac{\lambda}{M}}, & \mbox{ if } |t| > \lambda + M.\end{cases}
\end{eqnarray}
When $M\rightarrow \infty$, $\Theta_B$ becomes the soft-thresholding. When $M=0$, $\Theta_B$ reduces to the ridge thresholding. The penalty constructed from \eqref{constrP} is given by
\begin{eqnarray*}
P(\theta;\lambda, M) = \begin{cases} \lambda |\theta|, & \mbox{ if } |\theta|\leq M\\ \lambda \frac{\theta^2 + M^2}{2M}, & \mbox{ if } |\theta|>M, \end{cases}
\end{eqnarray*}
which is exactly the  `\textbf{Berhu}' penalty \cite{Owen} whose composition reverses that of \textbf{Huber}'s robust loss function.
The Berhu penalty on the singular values provides a \emph{convex} fusion of  the nuclear norm penalty  and the Frobenius norm (squared) penalty  in the problem of \eqref{oriprob-nodesign}. Unlike the elastic net \cite{ZouHas}, this   fusion is nonlinear and fully preserves the nondifferentiable behavior (around  zero) of the nuclear norm.

\paragraph*{A fusion between rank and Frobenius norm}
A direct thresholding rule that fuses the hard-thresholding and the ridge-thresholding is the hard-ridge thresholding \cite{SheTISP}
\begin{eqnarray}
\Theta_{HR}(t;\lambda,\eta)=\begin{cases} 0, & \mbox{ if } |t|< \lambda\\ \frac{t}{1+\eta}, & \mbox{ if }  |t|\geq\lambda. \end{cases} \label{hybridthfunc}
\end{eqnarray}
Setting $q\equiv 0$ in Proposition \ref{uniqsol-mat}, we obtain one associated penalty
\begin{eqnarray*}
P(\theta;\lambda, \eta)=\begin{cases} -\frac{1}{2} \theta^2 + \lambda |\theta|, & \mbox{ if } |\theta| < \frac{\lambda}{1+\eta}\\ \frac{1}{2} \eta \theta^2 +\frac{1}{2}\frac{\lambda^2}{1+\eta}, & \mbox{ if } |\theta| \geq \frac{\lambda}{1+\eta}. \end{cases} \label{hybridpenfunc}
\end{eqnarray*}
Interestingly, noticing that $\Theta_{HR}$ is discontinuous at $\lambda$, we can choose
\begin{align}
q(\theta;\lambda, \eta)=\begin{cases} \frac{(1+\eta)(\lambda-|\theta|)^2}{2} , & \mbox{ if } 0 < |\theta| < \lambda\\
0,  & \mbox{ if }  \theta=0 \mbox{ or } |\theta| > \lambda  \end{cases} \label{defofq}
\end{align}
which leads to $P(\theta)=\frac{1}{2} \eta \theta^2  + \frac{1}{2}\frac{\lambda^2}{1+\eta} 1_{\theta\neq0}$. Therefore, $\Theta_{HR}(\cdot; \lambda, \eta)$ can deal with the following rank-Frobenius penalty  in \eqref{oriprob-nodesign}
\begin{align}
\frac{1}{2} \eta \|\bsbB\|_F^2  + \frac{1}{2}\frac{\lambda^2}{1+\eta} \mbox{rank}(\bsbB). \label{rankfropen}
\end{align}
This penalty may be of interest in statistical learning tasks that have joint concerns of accuracy and parsimony: the rank portion enforces high rank deficiency, while the ridge (Frobenius) portion
shrinks $\bsbB$ to compensate for large noise and decorrelates the input variables  in large-$p$ applications.

At the end of this subsection, we present a perturbation result which will be used to establish the main result in the next subsection.
\begin{proposition}
\label{unifuncopt-mat}
Given $\bsbY\in {\mathbb R}^{n\times m}$, let $Q(\bsbB)=  \|\bsbY-\bsbB\|_F^2/2 + \sum P_{\Theta}(\sigma_{i}^{(\bsbB)}; \lambda)$, where $P_{\Theta}$ is the penalty obtained from \eqref{constrP}. Denote by $\hat \bsbB$  the  minimizer of $Q(\bsbB)$. 
Then for any matrix $\bsbDelta\in {\mathbb R}^{n\times m}$
\begin{eqnarray*}
Q(\hat \bsbB+\bsbDelta)-Q(\hat \bsbB) \geq \frac{C_1}{2} \|\bsbDelta\|_F^2,
\end{eqnarray*}
where  $C_1 =  1-{\mathcal L}_\Theta \geq 0$.
\end{proposition}
See Appendix \ref{approofpert} for its proof.

\subsection{Singular-value penalized vector GLMs}

In this subsection, we generalize the  results obtained for matrix approximation to vector GLMs.


Let $\bsbY=[\bsby_1, \cdots, \bsby_m]\in {\mathbb R}^{n\times m}$ be the response matrix with $m$ response variables and $n$ samples for each. Assume $y_{i,k}$ are independent and each follows a distribution in the natural exponential family $f(y_{i,k}; \theta_{i,k}) = \exp (y_{i,k}\theta_{i,k}-b(\theta_{i,k})+c(y_{i,k}))$, where $\theta_{i,k}$ is the natural parameter. Let $L_{i,k}=\log f(y_{i,k}, \theta_{i,k})$, $L=\sum_{k} \sum_{i} L_{i,k}$.
The {canonical link} function $g=(b')^{-1}$  is applied throughout the paper.
Let the model matrix  and the corresponding coefficient matrix be
\begin{align}
\begin{split}
\bsbX&=\left[\bsbx_1, \cdots, \bsbx_n\right]^T=\left[\tilde \bsbx_0, \tilde \bsbx_1, \cdots, \tilde \bsbx_p\right]\in {\mathbb R}^{n\times (p+1)} \mbox{ and }\\
 \bsbB&=[\bsbb_1, \cdots, \bsbb_m]=[\tilde \bsbb_0, \tilde \bsbb_1, \cdots, \tilde \bsbb_p]^T\in {\mathbb R}^{(p+1)\times m}, \end{split} \label{predsplit}
\end{align}
respectively. If $\tilde \bsbx_0 = \bsb{1}$,  $\tilde \bsbb_0$ represents the intercept vector.   For convenience, we use $\bsbB^{\circ}= [\tilde \bsbb_1, \cdots, \tilde \bsbb_p]^T$ to denote the submatrix of $\bsbB$  obtained by deleting  the first row $\tilde \bsbb_0^T$, and use $\bsbX^\circ$ to denote the submatrix of $\bsbX$ obtained by deleting the first column $\tilde \bsbx_0$.  
Given  any GLM with coefficients  $\bsbb$, we introduce
\begin{align*}
\mu(\bsbb)&\triangleq[g^{-1}(\bsbx_i^T\bsbb)]_{n\times 1} \mbox{  and  }
\bsbmI(\bsbb)&\triangleq  \bsbX^T\bsbW  \bsbX=  \bsbX^T\mbox{diag}\left\{b''(\bsbx_{i}^T\bsbb)\right\}_{i=1}^n \bsbX
\end{align*}
to denote the mean vector and the 
information matrix 
at $\bsbb$. For the $m$-response vector GLM, the mean matrix $\bsbmu(\bsbB)=\left[\mu_{i,k}\right]_{n\times m}$ is defined as $[\bsbmu(\bsbb_1), \cdots, \bsbmu(\bsbb_m)]$. 

\paragraph{Remarks.} (i)
Having $\tilde \bsbx_0$ and $\tilde \bsbb_0$ is necessary. For non-Gaussian GLMs,  one cannot center the response variables  because this may violate the distributional assumption. (ii)
For clarity, the above setup does not include any dispersion parameter. But all discussions in this subsection can be trivially extended to the exponential dispersion family $f(y_{i,k}; \theta_{i,k}, \phi) = \exp \{(y_{i,k}\theta_{i,k}-b(\theta_{i,k}))/a(\phi)+c(y_{i,k}, \phi)\}$ which  covers the vector Gaussian regression.

Our goal is to minimize \eqref{penform} or more generally, the following objective function
\begin{eqnarray}
F(\bsbB) &\triangleq& -\sum_{k=1}^m \sum_{i=1}^n L_{i,k}(\bsbb_k;  \bsbx_i, y_{ik}) + \sum_{s=1}^{p\wedge m} P(\sigma_{s}^{(\bsbB^{\circ})}; \lambda) \label{oriprob-glm}
\end{eqnarray}
for a large family of penalty functions  (possibly nonconvex).  The penalty is \emph{not} imposed on  $\tilde \bsbb_0$.

We construct the following sequence of iterates for solving the problem:
given $\bsbB^{(j)}$, perform the update
\begin{align}
\begin{cases}
\bsbB^{\circ (j+1)} & =  \Theta^{\sigma}(\bsbB^{\circ (j)}+ \bsbX^{\circ T} \bsbY - \bsbX^{\circ T} \bsbmu(\bsbB^{(j)}); {\lambda}), \\
\tilde  \bsbb_0^{(j+1)} & =  \tilde\bsbb_0^{(j)}  +   (\bsbY - \bsbmu(\bsbB^{(j)}))^T \tilde \bsbx_{0}.
\end{cases}
\label{tisp-mat}
\end{align}

\begin{theorem}
\label{conv_mat}
Given an arbitrary thresholding rule $\Theta$, let  $P(\cdot)$   be any function satisfying \eqref{constrP}. Starting with any $\bsbB^{(0)}\in {\mathbb R}^{(p+1)\times m}$, run \eqref{tisp-mat} to obtain a sequence   $\{\bsbB^{(j)}\}$. Denote by $A_k$ the set of $\{t\bsbb_k^{(j)}+(1-t)\bsbb_k^{(j+1)}: t \in (0, 1), j=1, 2, \cdots\}$, $1\leq k \leq K$, and define $$\rho=\max_{1\leq k \leq m} \sup_{\bsbxi_k\in A_k}  \|\bsbmI(\bsbxi_k)\|_2.$$
Suppose $\rho <  {2-{\mathcal L}_\Theta}$. Then
$F(\bsbB^{(j)})$ is decreasing and satisfies
\begin{eqnarray}
F(\bsbB^{(j)})- F(\bsbB^{(j+1)})\geq C \|\bsbB^{(j)}-\bsbB^{(j+1)}\|_F^2/2, \quad j=1,2,\cdots \label{asympreg_mat}
\end{eqnarray}
where $C= 2-{\mathcal L}_\Theta - \rho$.
Any limit point $\bsbB^{*}$ of the sequence $\bsbB^{(j)}$, referred to as a $\Theta^{\sigma}$-estimator,  is
a solution to the following equation:
\begin{eqnarray}
\begin{cases}
\bsbB^{\circ}  = \Theta^{\sigma}(\bsbB^{\circ} + \bsbX^{\circ T} \bsbY - \bsbX^{\circ T} \bsbmu(\bsbB); {\lambda})\\
\bsb{0}  =(\bsbY - \bsbmu(\bsbB))^T \tilde  \bsbx_0,
\end{cases}
\label{theta-mat}
\end{eqnarray}
under the assumption that $\Theta$ is continuous at all singular values of  $\bsbB^{\circ *} + \bsbX^{\circ T} \bsbY - \bsbX^{\circ T} \bsbmu(\bsbB^*)$.
\end{theorem}

The proof details are given in Appendix \ref{approofpenalg}.

Recall that $\mathcal L_{\Theta}\leq 1$. In implementation, we can scale the model matrix  by $\bsbX / k_0$ for any $k_0\geq  \sqrt{\rho}$ regardless of $\Theta$, and then perform \eqref{tisp-mat}. The $\Theta^{\sigma}$-estimate,   obtained  on the scaled data, can be scaled back to give an estimate on the original $\bsbX$.
The procedure  has a theoretical guarantee of  convergence and  \eqref{asympreg_mat} yields a good stopping criterion based on the change in   $\bsbB^{(j)}$. Empirically, we always observe $\bsbB^{(j)}$ has a unique limit point.
Similar to \citeasnoun{SheGLMTISP}, when it is possible to explicitly calculate the curvature parameter $\mathcal L_{\Theta}$, say for SCAD or soft-thresholding, we recommend using the smallest possible value of $k_0=\sqrt{{\rho}/{(2-\mathcal L_{\Theta})}}$, which   significantly speeds the convergence of the algorithm based on extensive experience. (For example, with a  convex penalty  we can set $k_0=\sqrt{{\rho}/{2}}$.)
We give two typical situations to find an upper bound for $\rho$  in theory.



\begin{example}[\textbf{Penalized Vector Gaussian GLM}] \label{exmultigauss}
For Gaussian regression, we can ignore the intercept term (after centering both responses and predictors beforehand), and the objective function \eqref{oriprob-glm} becomes
\begin{eqnarray}
\|\bsbY -\bsbX \bsbB\|_F^2/2 + \sum_{s=1}^{p\wedge m} P(\sigma_{s}^{(\bsbB)}; \lambda). \label{oriprob-glm-gauss}
\end{eqnarray}
\eqref{tisp-mat} reduces to
\begin{align}
\bsbB^{(j+1)} & = \Theta^\sigma (\bsbB^{(j)} + \bsbX^T\bsbY - \bsbX^T \bsbX \bsbB; \lambda ). \label{gaussmat}
\end{align}
Here, $\bsbmI=\bsbX^T \bsbW\bsbX=\bsbX^T\bsbX$. According to the theorem, $k_0$ can be chosen to be $\|\bsbX\|_2$ {regardless} of the thresholding rule and the penalty, where $\| \cdot \|_2$ denotes the spectral norm.

In the special case of imposing a direct rank penalty, where $\sum_{s=1}^{p\wedge m} P(\sigma_{s}^{(\bsbB)}; \lambda) = \frac{\lambda^2}{2} rank(B)$, another computational procedure based on the classical  RRR  algorithm   can be used. In fact, RRR studies a relevant but different problem,  with no penalty but subject to a low rank constraint. But we can  adapt the procedure  to minimizing $\|\bsbY-\bsbX\bsbB\|_F^2/2 + \lambda^2/2 \cdot \mbox{rank}(\bsbB)$ as follows (cf. \citeasnoun*{BuneaSheWegkamp}).
Suppose $\bsbX^T\bsbX$ is nonsingular and  $\bsbH$ is the hat matrix $\bsbX(\bsbX^T\bsbX)^{-1}\bsbX^T$.
(i) Apply spectral decomposition to $\bsbY^T \bsbH\bsbY$: $\bsbY^T \bsbH\bsbY=\bsbV \bsbD^2 \bsbV^T$ where $\bsbD=\mbox{diag}\{d_{1}, \cdots, d_m\}$ with $d_1\geq d_2\geq \cdots \geq d_m\geq 0$. (ii) Given any value of $\lambda$, define $r=\max\{i: d_i\geq \lambda\}$ and $\bsbV_r=\bsbV[\ ,1\mbox{:}r]$, by taking the first $r$ columns in $\bsbV$. (iii) Then the (globally) optimal solution is given by $$\hat\bsbB(\lambda) = (\bsbX^T\bsbX)^{-1}\bsbX^T \bsbY  {\mathcal P}_{\bsbV_r}=(\bsbX^T\bsbX)^{-1}\bsbX^T \bsbY \bsbV_r \bsbV_r^T,$$ where ${\mathcal P}_{\bsbV_r}$ is the orthogonal projection onto the column space of $\bsbV_r$.
We can show the $\Theta^{\sigma}$-estimate defined by \eqref{gaussmat} reduces to the RRR estimate in this case, the proof details given in  Appendix \ref{approofequiv}.
\begin{proposition}\label{redrrequiv}
Suppose $\bsbX\in {\mathbb R}^{n\times p}$ ($n\geq p$) has full column rank and $\|\bsbX\|_2\leq 1$.
Then the RRR estimate $\hat\bsbB(\lambda)$ constructed above must satisfy the $\Theta^\sigma$-equation
$\hat\bsbB = \Theta_H^\sigma(\hat\bsbB + \bsbX^T \bsbY - \bsbX^T \bsbX \hat\bsbB;\lambda)$ for matrix hard-thresholding $\Theta_H^{\sigma}$.
\end{proposition}

Unlike RRR, our  algorithm and convergent analysis do not require $\bsbX$ to have full rank or $n>p$. In comparison with the large-$p$ RSC \cite*{BuneaSheWegkamp}, \eqref{tisp-mat}  applies to any $\Theta$ (and covers all vector GLMs).
\end{example}

\begin{example}[\textbf{Penalized Vector Logistic GLM}]
Assume a classification setup where $y_{ik}$ are all binary. The singular-value penalized vector logistic regression minimizes 
\begin{eqnarray}
-\sum_{k=1}^m \sum_{i=1}^n \left( y_{i,k}\bsbx_i^T \bsbb_k - \log(1+\exp(\bsbx_i^T \bsbb_k))\right) + \sum_{s=1}^{p\wedge m} P(\sigma_{s}^{(\bsbB^\circ)}; \lambda). \label{oriprob-glm-logistic}
\end{eqnarray}
The first iteration step in \eqref{tisp-mat} becomes
\begin{align}
\bsbB^{^\circ(j+1)} & = \Theta^\sigma (\bsbB^{^\circ(j)} + \bsbX^{^\circ T}\bsbY - \bsbX^{^\circ T} \left[ {1}/({1+\exp(-\bsbx_i^T \bsbb_k^{(j)})})\right]_{n\times m}; \lambda ). \label{logitmat}
\end{align}
In {\tt R}, the  matrix $\bsbmu(\bsbB)$ can be simply constructed by {\tt 1/(1+exp(-X\%*\%B))}. 
Since $\bsbW(\bsbb)=\mbox{diag}\{b''(\bsbx_i^T\bsbb)\}=\mbox{diag}\{\pi_i(1-\pi_i)\} \preceq\bsbI/4$ with $\pi_i = 1/ (1+\exp(-\bsbx_i^T \bsbb))$, a crude but general choice of the scaling constant is  $k_0\geq \|\bsbX\|_2/2$, again, regardless of $\Theta$ and $\lambda$. Yet in applying a convex penalty such as the nuclear norm penalty, we can use $k_0= \| \bsbX \|_2/(2\sqrt 2)$ to speed the convergence.   \\
\end{example}

\paragraph*{Some related works}
There has been a surge of interest in nuclear norm penalization recently, in which case the penalty in \eqref{oriprob-glm} simplifies to a multiple of the sum of all singular values of $\bsbB^\circ$ or $\lambda \| \bsbB^\circ\|_*$. This gives a \emph{convex} optimization problem. In the statistics community, \citeasnoun**{YuanNNP} seem to be the first to study the nuclear norm penalized least squares estimator. A popular equivalent formulation of the nuclear norm minimization in optimization is through    semidefinite programming (SDP)~\cite{Fazel}.
See, e.g., \citeasnoun{CandRecht}, \citeasnoun{CandTao}, and \citeasnoun**{Ma} for some recent theoretical and computational achievements.

Although the nuclear norm provides a convex relaxation to the rank penalty,  this approximation  works only under certain regularity conditions (e.g., \citeasnoun{candplan}). \citeasnoun*{BuneaSheWegkamp} show that  direct rank penalization  achieves the same oracle rate in a much less restrictive manner. Yet in addition to the reduced rank regression studies (see  Example \ref{exmultigauss}), there have been very few attempts to extend the rank penalization beyond the Gaussian framework.
Two commonly cited works are \citeasnoun{yee} and \citeasnoun{copulas}. 
See Section \ref{sec:bg} for their limitations.
In comparison with these works, our  matrix thresholding algorithm  has a theoretical guarantee of convergence, is simple to implement, and covers a wide family of penalty functions as well as loss functions.

Finally, we point out a major difference between the thresholding-based iterative selection procedures (TISP) \cite{SheTISP} and the proposed algorithm which can be referred to as \emph{matrix-TISP}.
TISP aims for variable selection in a single-response model,  while here we discuss singular value regularization in vector GLMs. The singular-value sparsity or low rankness, different than   coefficient sparsity, offers a new type of parsimony that can be used for supervised feature extraction. It   brings a true multivariate flavor into our analysis.


\paragraph*{Feature extraction}
In many high-dimensional problems, \textit{feature extraction}, by transforming the input variables and creating a reduced set of new features, is a useful technique for dimension-reduction.
For example, PCA considers linear projections of correlated variables to construct new orthogonal features ordered by decreasing variances.
For singular-value penalized models, once a low-rank estimate $\hat\bsbB$ is obtained, one can attain the same goal.
Suppose the rank of $\hat\bsbB$ is $r$.
A direct way is to apply the reduced form SVD on $\hat\bsbB$, getting $\hat\bsbB=\bsbU \bsbD \bsbV^T$ with $\bsbD$ an ${r\times r}$ diagonal matrix. Next,   construct a new model matrix
\begin{align}
\mbox{\emph{Type-I:} } \quad \bsbZ\triangleq \bsbX\bsbU \ (\mbox{or } \bsbX\bsbU\bsbD) \label{type-I}
\end{align}
which has   only $r$ new predictors. We refer to this as  \emph{Type-I  extraction}. It can  be used for parameter tuning later.

On the other hand, it may be preferred to work on $\bsbX\hat\bsbB$ in some situations. Perform the spectral decomposition $\hat\bsbB^T\bsbX^T\bsbX\hat \bsbB = \bsbV \bsbD \bsbV^T$, where $\bsbV$ is an $m\times r$ orthogonal matrix. It follows that $\bsbX \hat\bsbB = \bsbX \hat\bsbB \bsbV \bsbV^T$.   Therefore, for the new design matrix $\bsbZ$ defined by
\begin{align}
\mbox{\emph{Type-II (or Post-Decorrelation)}: }  \bsbZ \triangleq \bsbX (\hat\bsbB \bsbV)\in {\mathbb R}^{n\times r}, \label{type-II}
\end{align}
each  column ($z$-predictor) can be represented as a linear combination of the columns of $\bsbX$, and the $r$ newly obtained   $z$-predictors are uncorrelated with each other, i.e., $\bsbZ^T\bsbZ$ is diagonal. We refer to this as \emph{Type-II  extraction} or \emph{post-decorrelation}. Type-I and Type-II are   not equivalent   in general (but coincide for the RRR estimate). The (linear) feature extraction is supervised 
and the corresponding dimension reduction can be dramatic when $r$ is much smaller than $p$.

\paragraph*{Initial point}
When the problem \eqref{oriprob-glm} is convex, we can further show (based on Theorem \ref{conv_mat}) that any $\Theta^\sigma$-estimate is a global minimum point. In this case,  the choice of the initial estimate $\bsbB^{(0)}$ is not essential and a pathwise algorithm with warm starts can be used in computing the solution path $\hat\bsbB(\lambda)$ for a series of values of $\lambda$. However, for nonconvex problems   we do not have such global optimality  given any  initial point $\bsbB^{0}$. Although one can try multiple random starts,
we   found  that  empirically, simply setting $\bsbB^{(0)}$ to be the zero matrix 
leads to a  solution with   good statistical performance. Intuitively, this looks for a local optimum that is close to zero. Of course, other initialization choices are possible.


\paragraph*{Parameter tuning} The challenge still comes from   nonconvexity. Take  the rank penalty as an example. The solution path $\hat\bsbB(\lambda)$ is \emph{discontinuous}, while the optimal penalty parameter $\lambda$ (as a surrogate for the Lagrange multiplier in convex programmings) is a function of both the data $(\bsbX, \bsbY)$ and the true coefficient $\bsbB$. Therefore, plain cross-validation with respect to $\lambda$  does not seem to be appropriate, as slightly perturbed data may result in serious regularization parameter mismatches.  We propose to cross-validate the \emph{range space} of the low rank estimator 
(as a function of $\lambda$) and refer to it as   the projective cross-validation (\textbf{PCV}). In the following, we focus on the rank-Frobenius   penalty \eqref{rankfropen} and the associated hard-ridge thresholding rule \eqref{hybridthfunc} to describe the idea.
Let $\hat \bsbB$ be a $\Theta_{HR}^{\sigma}$ estimator obtained from \eqref{tisp-mat} and write it as $\left[\begin{array}{c}\hat{\tilde \bsbb}_0^T \\ \hat{\bsbB^{\circ}}\end{array}\right]$ following our previous notation   (cf. \eqref{predsplit}). For the penalized part $\hat\bsbB^\circ$, denote its rank by $r$ and its  SVD by $\hat\bsbB^\circ = \bsbU \bsbD \bsbV^T$ with $\bsbD\in {\mathbb R}^{r\times r}$. Construct $r$ new features (Type-I)   $\bsbZ^\circ =  \bsbX^\circ \bsbU$  and set $\bsbZ = [\tilde \bsbx_0, \bsbZ^\circ]=[\bsbz_1, \cdots, \bsbz_{n}]^T$.
Let $\hat \bsbC^{\circ} = \bsbD \bsbV^T$ and  $\hat \bsbC = [\hat{\tilde \bsbb}_0,  \hat{\bsbC}^{\circ T}]^T=[\bsbc_1, \cdots, \bsbc_m]\in {\mathbb R}^{(r+1)\times m}$.

\begin{proposition}
\label{rowspaceHR}
Under the condition on $\rho$ given in Theorem  \ref{conv_mat}, for any $\Theta_{HR}^{\sigma}$-estimator $\hat \bsbB$, $\hat \bsbC$ defined above is a Frobenius   penalized estimator associated with  new model matrix  $\bsbZ$, i.e.,
\begin{align}
 \hat \bsbC  \in  \arg \min_{\bsbC \in \mathbb R^{(r+1)\times m}} -\sum_{k=1}^m \sum_{i=1}^n L_{i,k}(\bsbc_k;   \bsbz_i, y_{ik}) + \frac{\eta}{2} \| \bsbC^\circ \|_F^2. \label{localridgeopt}
\end{align}
\end{proposition}
See its proof in Appendix \ref{approofhrlocal}. When  $\eta>0$ or $\bsbZ$ has full column rank, the optimization problem \eqref{localridgeopt} is strictly convex and so $\hat \bsbC$ is the unique optimal solution.

The proposition implies that once $\bsbU$ is extracted, we can simply use maximum likelihood estimation   on the projected data to obtain the rank penalized estimator, or    ridge penalized   estimation to obtain the rank-Frobenius  penalized estimator. There is no need to run  the more expensive reduced rank fitting algorithms.


We now state the $K$-fold PCV procedure for tuning the rank penalty parameter in  \eqref{penform}.
\begin{enumerate}
\item Run Algorithm \eqref{tisp-mat} on the \emph{whole} dataset for a grid of values for $\lambda$. The solution path is denoted by $\hat \bsbB(\lambda_l)$, $l=1,\cdots, L$.
\item
Obtain $L$ candidate models via \eqref{type-I}, each with a new model matrix   $\bsbZ(l) = \bsbX \bsbU(l)$, $1\leq l \leq L$.
\item
Compute the cross-validation error for each model. Concretely, partition the sample index set   into $K$ (roughly) even subsets ${\mathcal T}_1, \cdots {\mathcal T}_K$.
Given   $\bsbZ(l)$, fit a vector GLM on the data without  the subset indexed by ${\mathcal T}_k$, and evaluate its validation  error (measured by   deviance) on the left-out subset. In all,  $K$ maximum likelihood estimates  are obtained and their validation errors are summed up to yield the CV  error of the candidate model  $\bsbZ(l)$. Repeat this for all $l: 1\leq l \leq L$.
\item
Find the optimal model that minimizes the CV error.
\end{enumerate}
In the pursuit of a parsimonious model with very low rank, a BIC penalty term can be added to the CV error   \cite{SheGLMTISP}. This is necessary  in the  large-$p$ setup  \cite{chenchen}.

PCV is much more efficient than  CV because the more involved reduced rank fitting algorithm runs only  once  beforehand, rather than $K$ times  in the CV trainings. The ML fitting   in Step 3,   justified by Proposition \ref{rowspaceHR},   involves  very  few  predictors.   Another benefit of PCV is that   the parameter mismatch issue is eliminated and all  $K$ trainings are regarding the same model and feature space.

When there is an additional ridge parameter (cf. \eqref{rankfropen}), the procedure still applies, but a two-dimensional grid for $(\lambda, \eta)$ has to be used. Fortunately, according to our experience,   the statistical performance is not very sensitive to small changes in the ridge parameter and we can choose a sparse grid for it. Step 3 now fits a series of  $l_2$-penalized GLMs. But again, this type of problems  is smooth and  convex;   Newton-based algorithms are reasonably fast. Finally, we mention that PCV shares some similarities to  the selective cross-validation (SCV) proposed for variable selection \cite{SheGLMTISP}.

\section{Rank Constrained  Vector GLMs for Feature Space Reduction}
\label{sec:conGLM}
In this section, we study the reduced rank GLMs in constraint form (cf. \eqref{conform}).
For any ${r}\geq 1$ and $\eta\geq 0$,   the problem of interest is
\begin{eqnarray}
\min_{\bsbB \in \mathbb R^{(p+1)\times m}}  -\sum_{k=1}^m \sum_{i=1}^n L_{i,k}(\bsbb_k;  \bsbx_i, y_{ik}) + \frac{\eta}{2} \|\bsbB^{\circ} \|_F^2 \quad \mbox{ s.t. } \quad rank(\bsbB^\circ)\leq {r}, \label{oriprob-glm-constr}
\end{eqnarray}
The additional Frobenius norm penalty is   to handle collinearity. Again, neither the penalty nor the constraint is imposed on the first row of $\bsbB$.

We introduce a \emph{quantile thresholding rule} $\Theta^{\#}(\cdot; {r}, \eta)$ as a variant of the hard-ridge thresholding. Given $1\le {r}  \le p$ and $\eta\ge 0$, $\Theta^{\#}(\bsb{a};{r}, \eta): {\mathbb R}^p\rightarrow {\mathbb R}^p$ is defined for any $\bsb{a}\in {\mathbb R}^p$ such that the ${r}$ largest components of $\bsb{a}$ (in absolute value) are shrunk   by a factor of $(1+\eta)$ and the remaining components are all set to be zero.
In the case of ties, a random tie breaking rule is used.
The matrix version of $\Theta^{\#}$ is defined as
\begin{eqnarray*}
\Theta^{\#\sigma}(\bsbB;\lambda)\triangleq \bsbU\mbox{diag}\{\Theta^{\#}([\sigma_{i}^{(\bsbB)}]; {r}, \eta)\} \bsbV^T, \quad \forall \bsbB \in \mathbb R^{p\times m}
\end{eqnarray*}
where $\bsbU$, $\bsbV$, and $\mbox{diag}\{\sigma_{i}^{(\bsbB)}\}$ are obtained from the SVD of $\bsbB$: $\bsbB=\bsbU^T \mbox{diag}\{\sigma_{i}^{(\bsbB)}\} \bsbV$.

Then, a simple procedure similar to \eqref{tisp-mat} can be used to solve \eqref{oriprob-glm-constr}:
given $\bsbB^{(j)}$, perform the update
\begin{align}
\begin{cases}
\bsbB^{\circ (j+1)} & =  \Theta^{\#\sigma}(\bsbB^{\circ (j)}+ \bsbX^{\circ T} \bsbY - \bsbX^{\circ T} \bsbmu(\bsbB^{(j)}); {r}, \eta), \\
\tilde  \bsbb_0^{(j+1)} & =  \tilde\bsbb_0^{(j)}  +   (\bsbY - \bsbmu(\bsbB^{(j)}))^T \tilde \bsbx_{0}.
\end{cases}
\label{tisp-mat-constr}
\end{align}
Starting with any $\bsbB^{(0)}\in {\mathbb R}^{(p+1)\times m}$, denote the sequence  obtained from \eqref{tisp-mat-constr} by $\{\bsbB^{(j)}\}$. Let $F$ be the objective function   \eqref{oriprob-glm-constr}. Define $\rho$   as   in Theorem \ref{conv_mat}.
\begin{theorem}
\label{constrconv_mat}
If   $\rho \le  1$,
$F(\bsbB^{(j)})$ is decreasing and satisfies
\begin{eqnarray*}
F(\bsbB^{(j)})- F(\bsbB^{(j+1)})\geq (1-\rho) \|\bsbB^{(j)}-\bsbB^{(j+1)}\|_F^2/2,  
\end{eqnarray*}
 and $rank(\bsbB^{\circ (j)})\leq {r}$, $\forall j\geq 1$.
\end{theorem}

See  Appendix \ref{approofconstrcov} for its proof. The preliminary scaling of  $\bsbX/k_0$ for any $k_0\geq \sqrt{\rho}$ guarantees the convergence. Still, PCV can be used for parameter tuning and model selection although the obtained estimate may not be globally optimal due to nonconvexity.

\paragraph*{Rank penalty vs. rank constraint}
We have developed algorithms \eqref{tisp-mat} and \eqref{tisp-mat-constr} for solving   \eqref{penform} (or \eqref{oriprob-glm}) and \eqref{conform} (or \eqref{oriprob-glm-constr}), respectively. The obtained estimates are usually local optimizers of  the corresponding objective functions. 
However, in non-Gaussian GLM setups, we found that the nonconvexity of either problem can be very strong. For instance, there may exist many local optima all having the same rank but spanning different subspaces in $\mathbb R^p$.
In consideration of this, the penalized solution path $\hat\bsbB(\lambda)$ ($0\leq \lambda <+\infty$) may provide  more candidate models of certain rank (if existing) than the constrained solution path $\hat\bsbB({r})$ ($r=1, 2,\cdots,   p \wedge m$), which is advantageous in   the stage of parameter tuning. This phenomenon is often observed in datasets where $p$ is comparable to or larger than $n$. Note that typically the direct rank penalized $\hat\bsbB(\lambda)$   has no rank monotonicity.


On the other hand, computing the solution path for the penalty form  is often more time-consuming in  large-$p$ applications.
The path $\hat\bsbB(\lambda)$  has jumps.
Assuming no  prior knowledge of the appropriate interval for $\lambda$, one has to specify a large search grid  fine enough  to cover a reasonable number of candidate models.
By contrast, for the problem of constraint form,  we can set a small upper bound for ${r}$  in pursuing a low rank model (say, ${r} \leq 0.5 n \wedge p \wedge m$ could be good),  and the natural grid spacing is $1$. With the grid focusing on small values of ${r}$ (which amounts to applying large thresholds in the iteration steps), Algorithm \eqref{tisp-mat-constr} runs efficiently.

%

\paragraph*{Feature space reduction}
To combine the virtues of both approaches, we propose to solve the rank constrained problem to perform feature space reduction, and then run  the rank penalized algorithm in the reduced feature space. This is very helpful in large-$p$ applications.
A crude sketch is as follows. First, we set ${r}=\alpha n \wedge p \wedge m$ with $\alpha<1$ (e.g., $\alpha=0.5$) and solve \eqref{conform}. Using the estimate $\hat\bsbB(r)$, we execute Type-I  feature  extraction \eqref{type-I} to construct a new model matrix $\bsbZ=\bsbX  \bsbU_1(r)$ with only $r$ factors (in addition to the intercept).
Next, we turn to the  penalized problem \eqref{oriprob-glm} on  $(\bsbY, \bsbZ)$:  Get  the solution path from running  Algorithm \eqref{tisp-mat},  and tune the regularization parameters to find the optimal estimate,  denoted by $\hat \bsbB'(\lambda_o)$. Our final coefficient matrix estimate  is given by  $\bsbU_1(r) \hat \bsbB'(\lambda_o)$. A small number of new predictive features can be constructed (and decorrelated) based on \eqref{type-II}.

According to this scheme, the sample size of the reduced problem on $\bsbZ$   is    large relative to the reduced dimension.
It is not difficult to show that for $n>p$, the update in \eqref{tisp-mat} is essentially a contraction, and so Algorithm \eqref{tisp-mat} converges fast.

A crucial assumption here is that the rank of the true model, denoted by $r^*$, is very small, compared with   the sample size $n$.
This makes it possible to choose a safe rank constraint value ${r}$ in \eqref{conform}, which,  though possibly much less than $p$, is still much larger than the true $r^*$. Hence the computational cost of obtaining a solution path according to  Algorithm \eqref{tisp-mat} can be effectively reduced with little performance loss.
This idea shares  similarity with the variable screening \cite{fanlv} proposed in the context of sparse variable selection. In the process of screening,  all relevant variables should be kept, while in feature space reduction, only the necessary factors, being  linear combinations of the original predictors and  typically as few as a handful, are required to lie in  reduced feature space we project $\bsbX$ onto.

In implementation, we further adopt a  path-following (annealing) idea  to reduce computational load and avoid greedy reduction.
Define a cooling schedule $r(t)$ ($0\leq t \leq T$) with $r(0)=p$ and $r(T)=r$, where $r$ is an upper bound of the target rank.
We conduct \emph{progressive} feature space reduction as follows. (As aforementioned,  $\bsbZ^{\circ}$ refers to $\bsbZ$ without the first column, $\bsbB^{\circ}$ refers to $\bsbB$ without the first row, and $\bsbU^{\circ}$ refers to $\bsbU$  without the first row and the first column.)
\begin{enumerate}
\item Let $t\leftarrow 0$, $\bsbZ \leftarrow \bsbX$, $\bsbU \leftarrow \bsbI$.
\item Iterate until $r(t)\leq r$:
\begin{enumerate}
\item Set the rank constraint value to be $r(t)$ and perform the update \eqref{tisp-mat-constr} on $(\bsbY, \bsbZ)$ for at most $M$ times (with $M$ pre-specified);
\item Obtain the left singular vectors of the current slope estimate $\bsbB^{\circ}$, denoted by $\bsbU_1(r(t))$;
\item Let $\bsbZ^{\circ} \leftarrow \bsbZ^{\circ} \bsbU_1(r(t))$, $\bsbU^{\circ} \leftarrow \bsbU^{\circ} \bsbU_1(r(t))$;
\item $t\leftarrow t+1$.
\end{enumerate}
\end{enumerate}
At the end,  $\bsbZ$ is delivered as the new design, and the orthogonal matrix $\bsbU$ gives  the accumulated transformation matrix.

The previously described prototype
reduction scheme corresponds to   $r(t)=r$ for any $t$.  With an annealing algorithm design,   the dimensionality of the feature space keeps dropping; the  $\bsbB$ involved  in \eqref{tisp-mat-constr} has only $r(t)$ columns. A slow cooling schedule with a small number of $M$ is recommended. Based on our experiments, it is not too greedy and is usually computationally affordable for  large-$p$ problems.

\section{Data Examples}

We use two  real data examples to illustrate the proposed methodology for dimension reduction and supervised feature extraction.

\paragraph{Example 1.}
First,  we make a practical comparison of the rank penalized estimators from solving \eqref{penform} and the rank constrained estimators from solving \eqref{conform} by use of a {zipcode} dataset. The whole dataset (available at the website of  \citeasnoun**{ESL})  contains normalized handwritten digits in $16 \times 16$ grayscale images.
The  digits were originally scanned from envelopes by the U.S. Postal Service and have been deslanted and size normalized. The space of pixel predictors  is of dimension $256$. We standardized all such predictors. The intercept term is included in the model and is always unpenalized.
We introduced $m=9$ indicator response variables for digits 0-8, using 9 as the reference class.

The training set is large in comparison with $p$ and $m$ ($7291$ images). We chose a subset of $n=300$ at random in this experiment to compare the penalized solution path and the constrained solution path. No additional Frobenius-norm penalty was enforced. The prediction results of the estimates are shown in Table \ref{tab:miscls-zip}, evaluated on  2007 test observations.

\begin{table}[ht]
\centering
\setlength{\tabcolsep}{1.5mm}
\caption{\small{Rank constraint vs. rank penalty. Misclassification rates of the constrained and penalized reduced rank logistic regressions (RRL$^{(c)}$ and RRL$^{(p)}$) are shown for  the   {\tt zipcode} (sub)dataset where $p=257, n=300$.   The rank $r$ controls the \# of newly constructed features.  }
}\label{tab:miscls-zip}

\small{
\begin{tabular}{c ccc c }
\hline

\hline
$r$ & $1$ & $2$ & $3$ & $4$  \\
\hline

%
%

RRL$^{(c)}$ &  66.52\%   & 55.06\% & 38.47\% & 33.83\%  \\
RRL$^{(p)}$ &   66.52\%  & 55.06\% & {38.32\%}, {38.37\%} & 33.83\% \\

\hline
RRL$^{(c)}$+SVM &  59.24\%     & 47.48\% & 33.58\% & 30.79\% \\
RRL$^{(p)}$+SVM &   59.24\%     & 47.48\% & 33.63\%,  {33.58\%} & 30.79\% \\

\hline\hline

$5$ & $6$ &$7$ & $8$ & $9$  \\
\hline
24.86\%  & {22.27\%} & 21.33\% & 21.33\% & {20.43\%} \\
{24.81\%}, 24.86\% & 22.42\% & 21.47\%, 21.33\% & 21.33\% & 20.43\%, {19.13\%} \\

\hline
23.02\%  &  {20.38\%} & 20.43\% & 20.28\% & {18.53\%} \\
23.02\% &  21.08\% & 20.33\%, 20.43\% & {20.33\%}, 20.28\% & 18.53\%, {15.84\%} \\

\hline

\hline
\end{tabular}
}
\end{table}	

From the table, at certain values of $r$,  the rank penalty offered more candidate models along its solution path  than the rank constraint. Note that these rank-$r$ estimators may behave differently in prediction and feature extraction. For $p\sim n$ or $p>n$, this phenomenon is commonly seen.
With an appropriate parameter tuning strategy, the penalty form gives   better chances to achieve a low error rate.


Of course, this comes with a price in computation. In our experiment, the time for obtaining the RRL$^{(c)}$ path was less than one minute, while computing the RRL$^{(p)}$ path, with a $50$-point grid for $\lambda$, took about four minutes.

There is no obligation  to predict through the obtained estimator; perhaps more useful is the much lower dimensional feature space  yielded from such an estimator.
Fancier classifiers such as SVM can be applied with the new features automatically extracted and decorrelated via \eqref{type-II}, and result in lower error rates as shown in the table.

Finally, we add a comment that in some situations  there may exist  no penalized solution at certain  rank values. Yet with a large $\lambda$-grid chosen, the performance of the penalized estimator (after parameter tuning) does not seem to be worse than that of the constrained estimator.

\paragraph*{Example 2.} The \emph{Computer Audition Lab 500} ({\tt CAL500}) dataset is collected by \citeasnoun**{Turnbull_SemanticAudio} and involves $502$ Western popular songs by different artists selected from the past 50 years. Digital audio files were played to students to annotate these songs with $m=174$ words representing emotion, genre, instrument, vocals, etc. The concepts  characterized by the words are not   mutually exclusive and one song can be annotated with multiple labels. This is called  \emph{multi-label} data in machine learning.
The predictors are {MFCC-Delta}  audio features from  analyzing a short-time segment of the audio signal.  \citeasnoun**{Turnbull_SemanticAudio} used $68$ such feature vectors.
To allow for interactions between these audio features and to make   a more challenging problem, we consider a full quadratic model including all main effects, quadratic effects, and pairwise interactions. Hence $p=68+68(69)/2+1=2415$. We  split the data into two halves and used $n=251$ songs for training and the other $251$ for testing.

For this small-sample-size-high-dimensional problem, the SVM using all  ${2415}$ predictors gave a total misclassification rate of $21.2\%$, which is not all bad.  On the other hand,  the proposed reduced rank methodology can be applied to automatically construct new predictive audio features, possibly much fewer than $2415$.  The supervised nature is important because only the audio features helpful in annotation (classification) are  truly meaningful in this learning task.

First, we conducted the progressive feature space reduction  introduced in Section \ref{sec:conGLM}, with the upper bound of the target rank set to be $20$. Then we ran Algorithm  \eqref{tisp-mat} to fit a penalized reduced rank vector logistic regression with the $20$ extracted features. The rank-Frobenius penalty was chosen due to serious collinearity arising from the high-dimensional quadratic model. The parameters were tuned by $5$-fold PCV  with BIC correction.

Surprisingly, our final estimate $\hat\bsbB$ has $rank(\hat\bsbB{^\circ})=2$, which gives a dramatic  dimension reduction from $2514$ to $2$.  But the SVM trained based on just the two new features   yielded   an improved error rate of $14.13\%$. In fact, even using the vanilla reduced rank estimator, we can achieve an error rate of  $14.36\%$. The   per-word precision and recall (cf.   \citeasnoun**{Turnbull_SemanticAudio} for the detailed definitions) are, respectively, $35.6\%$  and $8.7\%$ on the test dataset,  comparable  to the rates of the  two advocated approaches in \citeasnoun**{Turnbull_SemanticAudio}.
But our model is more parsimonious and creates two concise audio summary indexes for semantic annotation.


\section{Conclusion}
Supervised linear feature extraction can be obtained from a reduced rank vector model.
We studied  rank penalized and rank constrained generalized linear models and discussed how to adapt them to feature extraction and feature space reduction. The latter technique helps to reduce  the computational cost significantly in high dimensions.
We also noticed the strong nonconvexity of such problems  raises some serious  issues in  data-resampling based parameter tunings, but the proposed projective cross-validation works decently in general and is efficient.
Through reduced rank modeling, dimension reduction can be attained if the rank of the model is small relative to the number of predictors.
The work can be viewed as a supervised and parametric generalization of  the principle component analysis.

\appendix
\section{Proof of Proposition \ref{uniqsol-mat}}
\label{approofMatApp}
To prove Proposition \ref{uniqsol-mat}, we first introduce  two lemmas.

\begin{lemma}[\citeasnoun{vonNeumann}]
\label{vonNeu}
Let $\bsbA, \bsbB$ be two $n\times n$ matrices. Then
\begin{eqnarray}
|Tr(\bsbA\bsbB)| \leq \sum \sigma_i(\bsbA) \sigma_i(\bsbB),
\end{eqnarray}
where $\sigma_1(\bsbA) \geq \sigma_2(\bsbA)\geq \cdots \geq \sigma_n(\bsbA)$ and  $\sigma_1(\bsbB) \geq \sigma_2(\bsbB)\geq \cdots\geq \sigma_n(\bsbB)$ are ordered singular values of $\bsbA$ and $\bsbB$ respectively.
\end{lemma}
We refer to \citeasnoun{noteGrig} for an elementary proof.
\begin{lemma}
\label{uniqsol-gen}
Given a thresholding rule $\Theta$, let $P$ be any penalty  satisfying  condition \eqref{constrP} in Proposition \ref{uniqsol-mat}.  Then, the univariate minimization problem $\min_\theta  (t-\theta)^2/2 + P(\theta;\lambda)$ has a unique optimal solution $\hat\theta=\Theta(t;\lambda)$ for every $t$ at which $\Theta(\cdot;\lambda)$ is continuous.
\end{lemma}
\begin{proof}
Apply Lemma 1 in \citeasnoun{SheGLMTISP}.
\end{proof}

\noindent \textit{Proof of the optimality part of Proposition \ref{uniqsol-mat}.}
Let   $\bsbY\in {\mathbb R}^{n\times m}$ and assume $n\geq m$ without any loss of generality. Let $\bsbY=\bsbU_0 \bsbD_0 \bsbV_0^T$ and $\bsbB=\bsbU \bsbD \bsbV^T$ be the SVDs where  $\bsbD_0=\mbox{diag}(d_{0,i})$ and $\bsbD=\mbox{diag}(d_i)$ 
with  $d_{0,1}\geq d_{0,2}\geq \cdots \geq d_{0,m}$ and $d_{1}\geq d_{2}\geq \cdots\geq d_{m}$. Clearly,
\begin{eqnarray*}
\|\bsbY-\bsbB\|_F^2 &=& \|\bsbY\|_F^2 + \|\bsbB\|_F^2 -2 Tr(\bsbY^T \bsbB) \\
&=&  \|\bsbY\|_F^2 + \|\bsbB\|_F^2 -2 Tr([\bsbY \ \ \bsb{0}]^T [\bsbB \ \ \bsb{0}]),
\end{eqnarray*}
where $[\bsbB \ \ \bsb{0}]\in {\mathbb R}^{n \times n}$ and   $[\bsbY \ \ \bsb{0}]\in {\mathbb R}^{n \times n}$. 
It follows from Lemma \ref{vonNeu} that
$Tr(\bsbY^T \bsbB) \leq \sum d_{0,i} d_{i}$.
Hence
\begin{align}
F(\bsbB) &\geq  (\|\bsbD_0\|_F^2 + \|\bsbD \|_F^2 -2 Tr(\bsbD_0 \bsbD))/2 + \sum P(d_i; \lambda) \label{vonapplied}\\
&= \sum (d_{0,i}- d_i)^2/2 + \sum P(d_i ; \lambda).\notag
\end{align}
Now the problem reduces to
$$
\min_{d_i}  \sum (d_{0,i}- d_i)^2/2 + \sum P(d_i ; \lambda).
$$
The optimal solution $\hat \bsbB$ then follows from Lemma \ref{uniqsol-gen}.
\qed

The argument above only implies the singular values of $\hat\bsbB$ are unique (up to permutation).
Although one can possibly argue the uniqueness of $\hat\bsbB$ by studying
the condition under which equality is achieved in \eqref{vonapplied}, another formal proof of the uniqueness is deferred to  Appendix \ref{approofpert}.  

\section{Proof of Proposition \ref{unifuncopt-mat}}
\label{approofpert}

Let $\bsbB=\hat \bsbB + \bsbDelta$. Suppose $\bsbY=\bsbU_0 \bsbD_0 \bsbV_0^T$, $\hat \bsbB = \bsbU_0 \hat \bsbD \bsbV_0^T$,  and $\bsbB=\bsbU \bsbD \bsbV^T$ are the SVDs.
We have
\begin{eqnarray*}
&&\|\bsbY-\bsbB\|_F^2/2 - \|\bsbD_0 - \bsbD\|_F^2 /2 \\
& = & -Tr(\bsbB^T \bsbY) + Tr(\bsbD_0 \bsbD)\\
&=&-Tr(\bsbB^T (\bsbY-\hat\bsbB)) +  Tr((\bsbD_0 - \hat\bsbD) \bsbD) +   Tr(\bsbD \hat\bsbD)-Tr(\bsbB^T \hat\bsbB) \\
& = &  -Tr(\bsbV \bsbD \bsbU^T  \bsbU_0 (\bsbD_0 - \hat\bsbD)\bsbV_0^T) + Tr((\bsbD_0 - \hat\bsbD) \bsbD) +   Tr(\bsbD \hat\bsbD)-Tr(\bsbB^T \hat\bsbB)
\end{eqnarray*}
By Proposition \ref{uniqsol-mat} $\hat\bsbD \preceq \bsbD_0$, i.e., $\bsbD_0 - \hat\bsbD$ is positive semi-definite.
By augmenting $\bsbY-\bsbB$ and  $\bsbB$ and applying 
Lemma \ref{vonNeu}, we can prove
$$
Tr( \bsbD  (\bsbD_0 - \hat \bsbD)  \geq Tr(\bsbV \bsbD \bsbU^T \bsbU_0 (\bsbD_0 - \hat \bsbD) \bsbV_0^T ),
$$
from which it follows that
\begin{eqnarray*}
\|\bsbY-\bsbB\|_F^2/2 - \|\bsbD_0 - \bsbD\|_F^2 /2 &\geq &   Tr(\bsbD \hat\bsbD)-Tr(\bsbB^T \hat\bsbB) \\
&\geq&  C_1 (Tr(\bsbD\hat \bsbD) - Tr(\bsbB^T \hat \bsbB)).
\end{eqnarray*}
Now we have
\begin{eqnarray*}
Q(\bsbB)-Q(\hat \bsbB)&=&\|\bsbY-\bsbB\|_F^2/2 - \|\bsbY - \hat\bsbB\|_F^2/2 + \sum P_{\Theta}(d_{i}; \lambda) - \sum P_{\Theta}(\hat d_{i}; \lambda)\\
&\geq& \|\bsbD_0 - \bsbD\|_F^2 /2 - \|\bsbD_0 - \hat\bsbD\|_F^2/2 + \sum P_{\Theta}(d_{i}; \lambda) - \sum P_{\Theta}(\hat d_{i}; \lambda) \\
&& + C_1 Tr(\bsbD\hat \bsbD - \bsbB^T \hat \bsbB)\\
&=& \sum \left((d_{0,i} - d_i)^2/2 +   P_{\Theta}(d_{i}; \lambda)\right)- \left( (d_{0,i} - \hat d_i)^2/2 + P_{\Theta}(\hat d_{i}; \lambda)\right)\\
&& + C_1 Tr(\bsbD\hat \bsbD - \bsbB^T \hat \bsbB)\\
&\geq & C_1 \sum (d_i-\hat d_{i})^2/2 + C_1 Tr(\bsbD\hat \bsbD - \bsbB^T \hat \bsbB)\\
&=& C_1 (\|\bsbD-\hat\bsbD\|_F^2/2 + Tr(\bsbD\hat \bsbD) - Tr(\bsbB^T \hat \bsbB))\\
&=& C_1(\| \bsbD\|_F^2/2 + \|\hat\bsbD\|_F^2/2 - Tr(\bsbB^T \hat \bsbB))\\
&=& C_1 \|\bsbB - \hat \bsbB\|_F^2/2.
\end{eqnarray*}
The second inequality is due to Lemma 2 in \citeasnoun{SheGLMTISP}. \qed \\


\noindent \textit{Proof of the optimality part of Proposition \ref{uniqsol-mat}.}
From the comment in Appendix \ref{approofMatApp}, any optimal solution $\bsbB$ must have the same nonzero singular values (up to permutation) as $\hat\bsbB$, i.e., $d_i=\hat d_i$,  seen from the proof of Proposition \ref{uniqsol-mat}. A more careful examination of the proof of Proposition \ref{unifuncopt-mat}  shows $Q(\bsbB)-Q(\hat \bsbB) \geq Tr(\bsbD\hat \bsbD - \bsbB^T \hat \bsbB) = \|\bsbB - \hat \bsbB\|_F^2/2$. Therefore, the globally optimal solution $\hat\bsbB$ in Proposition \ref{uniqsol-mat} must be unique. \qed

\section{Proof of Theorem \ref{conv_mat}}
\label{approofpenalg}
The proof is similar to that of   Theorem 2.1 in \citeasnoun{SheGLMTISP}.
Define a surrogate function $G$ for any $\bsbA=[\bsba_1, \cdots, \bsba_m]=[\tilde \bsba_0, \tilde \bsba_1, \cdots, \tilde \bsba_{p}]^T$ and $\bsbB=[\bsbb_1, \cdots, \bsbb_m]\in {\mathbb R}^{(p+1)\times m}$
\begin{eqnarray}
G(\bsbB,\bsbA)&=& -\sum_{k=1}^m \sum_{i=1}^n L_{i,k}(\bsba_k) + \sum_{s=1}^{p\wedge m} P(\sigma_{s}^{(\bsbA^\circ)};\lambda)+ \frac{1}{2} \| \bsbA-\bsbB\|_F^2 \notag \\
&&- \sum_{k=1}^m \sum_{i=1}^n (b(\bsbx_{i}^T \bsba_k) - b(\bsbx_{i}^T\bsbb_k)) + \sum_{k=1}^m \sum_{i=1}^n \mu_{i,k} (\bsbx_{i}^T\bsba_k-\bsbx_{i}^T\bsbb_k),
\notag 
\end{eqnarray}
where  $\mu_{i,k}=g^{-1}(\bsbx_{i}^T\bsbb_k)=b'(\bsbx_{i}^T\bsbb_k)$. It can be shown that given $\bsbB$,  minimizing $G$ over $\bsbA$ is equivalent to
\begin{eqnarray*}
&\arg \min_\bsbA &\frac{1}{2} \left\|\bsbA - \left [\bsbB+\bsbX^T\bsbY-\bsbX^T\bsbmu(\bsbB)\right ]\right\|_F^2 +  \sum_{s=1}^{p\wedge m} P(\sigma_{s}^{(\bsbA^\circ)};\lambda). \label{optovergamma}
\end{eqnarray*}
By Proposition \ref{uniqsol-mat},
$\bsbB^{(j+1)}$ in \eqref{tisp-mat} can be characterized by $\arg \min_\bsbA G(\bsbB^{(j)}, \bsbA)$.
Furthermore, we have  for any $\bsbDelta \in {\mathbb R}^{(p+1)\times m}$
\begin{eqnarray}
G(\bsbB^{(j)}, \bsbB^{(j+1)}+\bsbDelta) - G(\bsbB^{(j)}, \bsbB^{(j+1)})\geq \frac{C_1}{2} \|\bsbDelta\|_F^2 
\label{gbound}
\end{eqnarray}
with $C_1= \max(0, 1-{\mathcal L}_{\Theta})$, by applying  Proposition \ref{unifuncopt-mat} and Lemma 1 in \citeasnoun{SheGLMTISP}, and noting that $q_s(\sigma^{(\bsbB^{\circ(j+1)})})=0$, for  $\bsbB^{\circ(j+1)}$ obtained by $\Theta^\sigma$-thresholding.

Next, Taylor series expansion  gives
\begin{eqnarray*}
&&F(\bsbB^{(j+1)})+\sum_k\frac{1}{2}(\bsbb_{k}^{(j+1)}-\bsbb_{k}^{(j)})^T(\bsbI- \bsbmI(\bsbxi_{k}^{(j)}))(\bsbb_{k}^{(j+1)}-\bsbb_{k}^{(j)})\\
& =  & G(\bsbB^{(j)},\bsbB^{(j+1)}) \leq G(\bsbB^{(j)},\bsbB^{(j)})-\sum_k\frac{C_1}{2}(\bsbb_{k}^{(j+1)}-\bsbb_{k}^{(j)})^T  (\bsbb_{k}^{(j+1)}-\bsbb_{k}^{(j)})\\
&=&F(\bsbB^{(j)})-\frac{C_1}{2}\| \bsbB^{(j+1)}-\bsbB^{(j)} \|_F^2.
\end{eqnarray*}
\eqref{asympreg_mat} can be obtained. In fact, this decreasing property  holds for any $\rho \leq  2-{\mathcal L}_\Theta$.

Let $\bsbB^{(j_l)}\rightarrow\bsbB^*$ as $l\rightarrow\infty$. Under the condition $\rho< 2-{\mathcal L}_\Theta$, $C$ is strictly positive  and
$
\| \bsbB^{(j_l+1)}-\bsbB^{(j_l)}\|_F^2/2 \leq (F(\bsbB^{(j_l)})- F(\bsbB^{(j_l+1)}))/C \leq (F(\bsbB^{(j_l)})- F(\bsbB^{(j_{k+1})}))/C\rightarrow 0.
$
That is, $\Theta^{\sigma}(\bsbB^{\circ(j_l)}+\bsbX^{\circ T} \bsbY -\bsbX^{\circ T} \bsbmu(\bsbB^{(j_l)});\lambda)-\bsbB^{\circ(j_l)}\rightarrow 0$. Therefore, $\bsbB^*$ is a solution to \eqref{theta-mat} due to the continuity assumption. 
\qed

\section{Proof of Proposition \ref{redrrequiv}}
\label{approofequiv}
Let $\bsbM=(\bsbX^T\bsbX)^{-1/2}\bsbX^T\bsbY$ and $r_0 = \mbox{rank}(\bsbM)$. Obviously, $r_0\leq p\wedge m$.
Note that $\bsbM^T\bsbM = \bsbY^T\bsbH\bsbY$. Assume $\bsbM = \bsbU  \bsbD \bsbV^T$ is the SVD of $\bsbM$ with $\bsbU\in{\mathbb R}^{p\times r_0}$, $\bsbV\in{\mathbb R}^{m \times r_0}$, and $\bsbD\in{\mathbb R}^{r_0\times r_0}$. Suppose  all (positive) diagonal entries of $\bsbD$ are arranged in decreasing order.
Let $\bsbA\triangleq \bsbX^T \bsbY - \bsbX^T \bsbX \hat\bsbB$. To prove $\hat\bsbB$ obeys the $\Theta^\sigma$-equation \eqref{theta-mat} for hard-thresholding, it suffices to show that (i) there exists a $p\times r_0$ orthogonal matrix $\bsbU_*$ satisfying $\bsbU_*^T \bsbU_*=\bsbI$  such that $\bsbU_*^T(\hat \bsbB\hat \bsbB^T)\bsbU_*$ and $\bsbU_*^T (\bsbA\bsbA^T) \bsbU_*$ are both diagonal; (ii) there exists an $m\times r_0$ orthogonal matrix $\bsbV_*$  such that $\bsbV_*^T(\hat \bsbB^T\hat \bsbB)\bsbV_*$ and $\bsbV_*^T (\bsbA^T\bsbA) \bsbV_*$ are both diagonal; (iii) $Tr(\hat\bsbB^T\bsbA)=0$; (iv) the singular values of $\bsbA$ are all bounded by $\lambda$.

Recall that $r=\max\{i: d_i\geq \lambda\}$ and $\bsbV_r=\bsbV[\ ,1\mbox{:}r]$. Introduce $\bsbV_{-r}=\bsbV[\ ,\mbox{(}r\mbox{+}1\mbox{):}r_0]$, formed by deleting the first $r$ columns in $\bsbV$. Then we have
\begin{align}
\bsbA &= \bsbX^T\bsbY - \bsbX^T \bsbX \hat\bsbB  =\bsbX^T\bsbY - \bsbX^T \bsbH \bsbY {\mathcal P}_ {\bsbV_r}\notag\\
&= \bsbX^T\bsbY (\bsbI- {\mathcal P}_ {\bsbV_r})= \bsbX^T\bsbY {\mathcal P}_ {\bsbV_{-r}} \notag\\
&=  \bsbX^T\bsbY   \bsbV_{-r}\bsbV_{-r}^T = (\bsbX^T\bsbX)^{1/2}\bsbM\bsbV_{-r}\bsbV_{-r}^T. \label{Aformula}
\end{align}
Obviously, $Tr(\hat\bsbB^T\bsbA)= 0$.  (iii) is true.
On the other hand, we can rewrite $\hat\bsbB$ as
\begin{align}
\hat \bsbB
&=(\bsbX^T\bsbX)^{-1/2} \bsbM \bsbV_r \bsbV_r^T = (\bsbX^T\bsbX)^{-1/2}\bsbM \bsbV \left[ \begin{array}{cc} \bsbI_{r\times r} & \\ & \bsb{0}_{(r_0-r)\times (r_0-r)}\end{array}\right] \bsbV^T \notag\\
& = (\bsbX^T\bsbX)^{-1/2} \bsbU \bsbD \left[ \begin{array}{cc} \bsbI_{r\times r} & \\ & \bsb{0}_{(r_0-r)\times (r_0-r)}\end{array}\right] \bsbV^T= (\bsbX^T\bsbX)^{-1/2} \Theta_H^{\sigma}(\bsbM; \lambda). \label{Bformula2}
\end{align}

Now we obtain
\begin{align}
\hat\bsbB^T\hat\bsbB &= \bsbV \left[ \begin{array}{cc} \bsbI_{r\times r} & \\ & \bsb{0}_{(r_0-r)\times (r_0-r)}\end{array}\right] \bsbD\bsbU^T (\bsbX^T \bsbX)^{-1}\bsbU \bsbD \left[ \begin{array}{cc} \bsbI_{r\times r} & \\ & \bsb{0}_{(r_0-r)\times (r_0-r)}\end{array}\right] \bsbV^T \label{BBformula}
\\
\bsbA^T\bsbA & = \bsbV \left[ \begin{array}{cc}\bsb{0}_{r\times r}  & \\ &\bsbI_{(r_0-r)\times (r_0-r)} \end{array}\right] \bsbD\bsbU^T (\bsbX^T \bsbX) \bsbU \bsbD \left[ \begin{array}{cc} \bsb{0}_{r\times r} & \\ & \bsbI_{(r_0-r)\times (r_0-r)} \end{array}\right] \bsbV^T. \label{AAformula}
\end{align}
(iv) is straightforward from \eqref{AAformula}:
\begin{align*}
\|\bsbA^T \bsbA\|_2 \leq \|  \bsbX \|_2^2 \cdot  \left\|\bsbU \bsbD \left[ \begin{array}{cc} \bsb{0}_{r\times r} & \\ & \bsbI_{(r_0-r)\times (r_0-r)} \end{array}\right] \bsbV^T\right\|_2^2 \leq 1\cdot d_{r+1}^2 \leq \lambda^2.
\end{align*}
\eqref{BBformula} + \eqref{AAformula} also implies (ii). In fact, introducing   $\bsbG = \bsbD\bsbU^T (\bsbX \bsbX)^{-1}\bsbU \bsbD$, $\bsbH = \bsbD\bsbU^T (\bsbX \bsbX) \bsbU \bsbD$,  $\bsbG_{11} = \bsbG[1\mbox{:}r,1\mbox{:}r]$, $\bsbH_{22} = \bsbH[\mbox{(}r\mbox{+}1\mbox{):}r_0,\mbox{(}r\mbox{+}1\mbox{):}r_0]$, and assuming the spectral decompositions of the two submatrices are given by $\bsbG_{11}=\bsbU_{11}^{G} \bsbD_{11}^{G}  (\bsbU_{11}^{G})^T$ and $\bsbH_{22}=\bsbU_{22}^{H}  \bsbD_{22}^{H}  (\bsbU_{22}^{H})^T$, respectively, then, $$\bsbV_*=\bsbV \left[ \begin{array}{cc} \bsbU_{11}^{G} & \\ & \bsbU_{22}^{H}\end{array}\right]$$ simultaneously diagonalizes $\hat\bsbB^T\hat\bsbB$ and $\bsbA^T\bsbA$ and satisfies $\bsbV_*^T  \bsbV_*=\bsbI$.

Finally, we construct  $\bsbU_*$ to prove (i). From \eqref{Bformula2} and \eqref{Aformula},
\begin{align}
\hat\bsbB\bsbV_* &= (\bsbX^T\bsbX)^{-1/2} \bsbU \bsbD \left[ \begin{array}{cc} \bsbI_{r\times r} & \\ & \bsb{0}_{(r_0-r)\times (r_0-r)}\end{array}\right] \left[ \begin{array}{cc} \bsbU_{11}^{G} & \\ & \bsbU_{22}^{H}\end{array}\right]\label{bv}\\
\bsbA\bsbV_* &= (\bsbX^T\bsbX)^{1/2} \bsbU \bsbD \left[ \begin{array}{cc} \bsb{0}_{r\times r} & \\ & \bsbI_{(r_0-r)\times (r_0-r)}\end{array}\right] \left[ \begin{array}{cc} \bsbU_{11}^{G} & \\ & \bsbU_{22}^{H}\end{array}\right]. \label{av}
\end{align}
Let $\tilde\bsbG=(\bsbX^T\bsbX)^{-1/2} \bsbU \bsbD$ and $\tilde\bsbH=(\bsbX^T\bsbX)^{1/2} \bsbU \bsbD$. Then $\tilde\bsbG^T\tilde\bsbG=\bsbG$, $\tilde\bsbH^T\tilde\bsbH=\bsbH$. By construction,  $\bsbU_{11}^{G}$ and $\bsbU_{22}^{H}$  must be the right-singular vectors of $\tilde \bsbG_{1}=\tilde\bsbG[\ ,1\mbox{:}r]$ and $\tilde \bsbH_{2}=\tilde\bsbH[\ ,\mbox{(}r\mbox{+}1\mbox{)}\mbox{:}r_0]$ respectively. Denoting by $\bsbU_1^{\tilde G}$ and $\bsbU_2^{\tilde H}$ their associated left-singular vectors respectively, 
we get
$$\bsbU_*=\left[ \begin{array}{cc}  \bsbU_1^{\tilde G} & \bsbU_2^{\tilde H}\end{array}\right]$$ which makes  both $\bsbU_*^T\hat\bsbB\hat\bsbB^T \bsbU_*$ and $\bsbU_*^T\bsbA\bsbA^T\bsbU_*$ diagonal. To prove $\bsbU_*$ is the desired matrix in  (i), it remains to show the orthogonality of $\bsbU_*$.
It follows from \eqref{bv} and \eqref{av} that
\begin{eqnarray*}
(\tilde \bsbG_1\bsbU_{11}^G)^T \tilde \bsbH_2 \bsbU_{22}^H &= & (\bsbU_{11}^G)^T  \left((\bsbX^T\bsbX)^{-1/2} \bsbU \bsbD \left[ \begin{array}{c} \bsbI_{r\times r} \\  \bsb{0}\end{array}\right]\right)^T  \cdot\\
&&   (\bsbX^T\bsbX)^{1/2} \bsbU \bsbD \left[ \begin{array}{c}\bsb{0} \\ \bsbI_{(r_0-r)\times (r_0-r)}  \end{array}\right]  \bsbU_{22}^H\\
&=&  (\bsbU_{11}^G)^T \left[ \begin{array}{cc} \bsbI_{r\times r} &  \bsb{0}\end{array}\right]  \bsbD \bsbU^T (\bsbX^T\bsbX)^{-1/2} (\bsbX^T\bsbX)^{1/2} \cdot\\
&&\bsbU \bsbD \left[ \begin{array}{c}\bsb{0} \\ \bsbI_{(r_0-r)\times (r_0-r)}  \end{array}\right] \bsbU_{22}^H\\
& = & (\bsbU_{11}^G)^T \left[ \begin{array}{cc} \bsbI_{r\times r} &  \bsb{0}\end{array}\right]  \bsbD \bsbU^T \bsbU \bsbD \left[ \begin{array}{c}\bsb{0} \\ \bsbI_{(r_0-r)\times (r_0-r)}  \end{array}\right] \bsbU_{22}^H\\
& = & (\bsbU_{11}^G)^T \left[ \begin{array}{cc} \bsbI_{r\times r} &  \bsb{0}\end{array}\right]  \bsbD^2 \left[ \begin{array}{c}\bsb{0} \\ \bsbI_{(r_0-r)\times (r_0-r)}  \end{array}\right] \bsbU_{22}^H\\
&=&\bsb{0}.
\end{eqnarray*}
Since $\bsbG_{11}$ and $\bsbH_{22}$ are positive definite (noting that $\bsbD\in{\mathbb R}^{r_0\times r_0}$ is nonsingular),  we further obtain $(\bsbU_1^{\tilde G})^T \bsbU_2^{\tilde H}=\bsb{0}$. Hence $\bsbU_*^T\bsbU_*=\bsbI$.
The proof is now complete. \qed

\section{Proof of Proposition \ref{rowspaceHR}}
\label{approofhrlocal}

From Theorem  \ref{conv_mat}, $\hat \bsbB$ satisfies
\begin{eqnarray}
\begin{cases}
\hat\bsbB^{\circ}  = \Theta_{HR}^{\sigma}(\hat\bsbB^{\circ} + \bsbX^{\circ T} \bsbY - \bsbX^{\circ T} \bsbmu(\hat\bsbB; \bsbX); \lambda, \eta)  \\
\bsb{0}  =(\bsbY - \bsbmu(\hat\bsbB; \bsbX))^T \tilde  \bsbx_0.
\end{cases}
\label{hr-eq}
\end{eqnarray}
Here, we write $\bsbmu(\hat\bsbB; \bsbX)$ to emphasize the dependence of the mean matrix on the design.
In this proof, we use the same submatrix notation as in Appendix \ref{approofequiv}.

Given the SVD $\bsbB^{\circ}=\bsbU \bsbD \bsbV^T$, by Definition \ref{def:thresholdmat}, there exist orthogonal matrices $\bar \bsbU$ and $\bar \bsbV$, as augmented versions of  $\bsbU$ and $\bsbV$, respectively, i.e., $\bsbU = \bar \bsbU[I, ]$ and $\bsbV = \bar \bsbV[I,]$ for some index set $I$, such that  $\hat\bsbB^{\circ} = \bar \bsbU \bsbSig \bar \bsbV^T$ and $\bsbX^{\circ T} \bsbY - \bsbX^{\circ T} \bsbmu(\hat\bsbB; \bsbX) = \bar \bsbU \bsbW \bar \bsbV^T$ are both the SVDs.  Clearly, $\bsbSig[I, I]=\bsbD$, $\bsbSig[I^c, I^c]=\bsb{0}$. Using the hard-ridge thresholding \eqref{hybridthfunc}, we rewrite the first equation in \eqref{hr-eq} as
\begin{align}
(1+\eta) \bsbB^\circ + \lambda (1+\eta) \bar \bsbU \bsbS  \bar \bsbV^T & = \bsbB^\circ + \bsbX^{\circ T} (\bsbY -  \bsbmu(\hat\bsbB; \bsbX)), \label{hr-eq2}
\end{align}
where $\bsbS$ is diagonal and satisfies $\bsbS[I, I]=\bsb{0}$ and $\bsbS[i, i] \leq 1$  for any $i\in I^c$.
Left-multiplying both sides of  \eqref{hr-eq2} by $\bsbU^T$ yields
\begin{align*}
\eta \bsbD \bsbV^T = \bsbU^T \bsbX^{\circ T} (\bsbY -  \bsbmu(\hat\bsbB; \bsbX)).
\end{align*}
On the other hand, from the construction of $\hat \bsbC$ and $\bsbZ$, it is easy to verify $\bsbx_i^T \hat \bsbB = \bsbz_i^T \hat \bsbC$, from which it follows that  $\mu(\hat\bsbB; \bsbX)=\mu(\hat\bsbC;\bsbZ)$. Therefore, $\hat\bsbC$ satisfies
\begin{align}
\begin{cases}
\eta \bsbC^\circ  &=   \bsbZ^{\circ T} (\bsbY -  \bsbmu(\hat\bsbC; \bsbZ))\\
\bsb{0} & =(\bsbY - \bsbmu(\hat\bsbC; \bsbZ))^T \tilde  \bsbz_0.
\end{cases} \label{hr-eq3}
\end{align}
Noticing that the optimization problem in \eqref{localridgeopt} is convex and \eqref{hr-eq3} gives its KKT equation, the conclusion follows. \qed

\section{Proof of Theorem \ref{constrconv_mat}}
\label{approofconstrcov}

\begin{lemma}
\label{uniqsol-mat-constr}
Given any $\bsbY\in {\mathbb R}^{n\times }$, $\hat \bsbB = \Theta^{\# \sigma}(\bsbY; {r}, \eta)$ is a globally optimal solution to
\begin{eqnarray}
\min_\bsbB  \frac{1}{2} \|\bsbY-\bsbB\|_F^2  + \frac{\eta}{2} \| \bsbB \|_F^2 \quad \mbox{s.t.} \quad rank(B)\leq {r} \label{oriprob-nodesign-constr}
\end{eqnarray}
\end{lemma}

\begin{proof}
The problem is equivalent to minimizing $$\frac{1+\eta}{2} \sum (\sigma_i^{(\bsbB)})^2-<\bsbY, \bsbB>$$ subject to $rank(\bsbB)\leq {r}$. Applying Lemma \ref{vonNeu} yields the result.
\end{proof}
The remainder of the proof  follows the same lines to the proof  of Theorem \ref{conv_mat}. See Appendix \ref{approofpenalg} for details.

\bibliographystyle{ECA_jasa}
\bibliography{tisp_matrix}

\end{document}